\documentclass[nonacm,screen,acmsmall]{acmart}
\AtBeginDocument{%
  }

\usepackage{tcolorbox}
\usepackage{multirow}
\usepackage{multicol}
\usepackage{colortbl}
\usepackage{pifont}
\usepackage{caption}

\begin{document}

\title{Gradient-free Task-Conditioned Retrieval for On-Device In-Context Learning}

\author{Xinyu Luo}
\email{xinyuluo8-c@my.cityu.edu.hk}
\orcid{0009-0007-5229-8733}

\author{Hui Liu}
\email{liuhui3@cityu.edu.hk}
\orcid{0000-0002-2279-7634}
\affiliation{%
  \institution{City University of Hong Kong}
  \city{Hong Kong SAR}
  \country{China}}

\author{Yihua Shao}
\email{yihuajerry@gmail.com}
\orcid{0009-0002-0475-7142}
\affiliation{%
  \institution{Institute of Automation Chinese Academy of Sciences}
  \city{Beijing}
  \country{China}
}

\author{Junyi Yang}
\email{junyiyang8-c@my.cityu.edu.hk}
\orcid{0000-0002-5867-4943}

\author{Arindam Basu}
\email{arinbasu@cityu.edu.hk}
\orcid{0000-0003-1035-8770}

\author{Haoliang Li}
\correspondingauthor
\email{haolia.li@cityu.edu.hk}
\orcid{0000-0002-8723-8112}
\affiliation{%
  \institution{City University of Hong Kong}
  \city{Hong Kong SAR}
  \country{China}}

\renewcommand{\shortauthors}{Luo et al.}

\begin{abstract}
On-device in-context learning (ICL) relies on pre-inference retrieval to select demonstrations for useful context before downstream model inference. This retrieval must exploit task-specific information while operating over local memories under limited computation, memory, and data-exposure budgets. We propose \emph{Conditional Retrieval Alignment} (CoRA), a gradient-free framework that converts a frozen encoder into a task-conditioned retriever using paired candidate inputs and outputs. CoRA selects complementary encoder layers, constructs an output-derived conditioning space from candidate memory, and aligns candidate input representations to this space through closed-form ridge regression. Low-rank factorization then produces a compact retrieval basis where candidate outputs are used only during offline index construction, whereas query-time retrieval requires only the query input and precomputed index. We show that CoRA's rank-constrained basis is the optimal low-rank compression of the output-conditioned fitted representation, and derive an exact two-pass streaming construction that avoids materializing the full fitted matrix. We further extend the framework to multimodal exemplar retrieval by incorporating visual representations into the conditioning and retrieval spaces. Experiments across ten textual datasets and four multimodal benchmarks with Llama-3.2-1B, MobileLLM-Pro, OpenFlamingo-3B, and Qwen3.5-2B, as well as end-to-end Raspberry Pi~5 deployment demonstrate that CoRA supports effective task-conditioned retrieval without retriever fine-tuning, backpropagation, or target-model calls.
\end{abstract}


\maketitle

\section{Introduction}\label{s1_intro}

Foundation models are increasingly deployed in settings that require efficient adaptation to diverse downstream tasks and user needs~\cite{kudugunta2024matformer,lin2020mcunet,zhang2024nomad}. One lightweight adaptation mechanism is \emph{test-time context construction}: a retriever selects a small set of textual or multimodal exemplars from a candidate pool to condition a model through in-context learning (ICL)~\cite{brown2020language,ding24ondevice,openclaw_github,zhao2026retrieval}. Because this selection occurs before the downstream model is invoked, we refer to it as \emph{pre-inference retrieval}. Retrieval quality can affect the utility and efficiency of ICL~\cite{an2023skill,lu2022fantastically,voronov2024mind,zhoubatch}. The problem is particularly relevant in personal and on-device deployments, where retrieval is performed over local memories containing potentially sensitive information and must operate under limited compute and memory budgets~\cite{gao2024aligning,shao2024privacylens,sun2024personalized,wang2023security,zhang2023privacy}. In this setting, practical retrieval systems must jointly consider the quality of selected exemplars, the cost of constructing their representations and index, and the latency of query-time search~\cite{leonhardt2024efficient}.

The central difficulty is that input similarity alone does not fully determine the utility of an ICL exemplar. Useful demonstrations should be related to the query while also reflecting task-specific regularities in their input-output relations~\cite{huang2024multimodal,li2025taco,rubin2022learning}. Existing approaches address this requirement in two main ways. One line adapts the retriever or representation space using task-specific learning signals~\cite{doveh2024towards,ye2023compositional,zhaommicl}. Another involves a stronger target model through rewriting, reranking, pseudo-labeling, or target-guided retrieval~\cite{chen2025provoking,li2023unified}. These approaches can capture task-specific behavior, but their training procedures, repeated model calls, or backbone modifications may be difficult to accommodate when retrieval must remain local. By contrast, lightweight retrievers such as BM25~\cite{robertson1995okapi} and frozen embedding similarity offer inexpensive search and compact representations~\cite{lin2023dense,zhao2024dense}, but they primarily rely on input-side similarity and do not explicitly use the output regularities available in paired candidate memories~\cite{alayrac2022flamingo,an2023skill,chen2025can,liu2022makes}. This leaves an open design tension between task-conditioned exemplar selection and efficient local retrieval.

To understand this tension, our prior conference work~\cite{liu2024unraveling} analyzed the mechanisms underlying learning-based demonstration selection. We found that their effectiveness can be related to two complementary signals: multi-level input similarity, which captures different types of lexical, syntactic, and semantic correspondence, and task-specific output similarity, which favors exemplars whose outputs are related to the desired task behavior. Based on these findings, that work introduced MLSM, which selects representative encoder layers and combines their task-agnostic similarity signals through lightweight optimization, and TTF, which injects task-specific information through retriever parameter updates. These methods established the importance of the two signals, but addressed them separately. MLSM does not explicitly condition its retrieval space on exemplar outputs, whereas TTF obtains task-specific information through parameter updates. The conference work therefore leaves open whether input-output information can be incorporated directly into a compact retrieval space without heavy retriever fine-tuning or target-model involvement.

This article addresses this question with \emph{Conditional Retrieval Alignment} (\textbf{CoRA}), a gradient-free framework for task-conditioned exemplar retrieval. CoRA selects complementary frozen-encoder layers, aligns their candidate-input representations to a conditioning space constructed from paired candidate inputs and outputs, and extracts a compact low-rank retrieval basis through closed-form linear algebra. Candidate outputs are used only when constructing the index. Query-time retrieval requires only the query input representation and the precomputed index. CoRA thus makes output-informed exemplar selection compatible with local pre-inference retrieval, without retriever parameter updates, backpropagation, or target-model calls. CoRA further admits an exact streaming index-construction procedure that avoids materializing the full fitted representation matrix, and it extends naturally to multimodal retrieval by incorporating visual features into the conditioning and retrieval spaces.

Our contributions are as follows:
\begin{itemize}
    \item We introduce CoRA, a gradient-free task-conditioned retrieval framework that uses paired candidate memories to align multi-level frozen input representations with output-derived structure and derive a compact low-rank retrieval basis in closed form, thereby enabling task-conditioned demonstration selection without retriever fine-tuning or target-model interaction.
    
    \item We derive an exact two-pass streaming realization of CoRA and characterize its construction-time working-memory cost. The procedure separates offline output-conditioned index construction from lightweight query-time projection and nearest-neighbor search, supporting local deployment over growing candidate pools.
    
    \item We develop CoRA-M, which extends the framework to multimodal exemplar retrieval by incorporating visual candidate representations into the output-conditioned alignment while retaining the text-layer selection mechanism, enabling a unified retrieval formulation across text-only and visual-language tasks.
    
    \item We evaluate retrieval effectiveness, construction efficiency, and deployment behavior across ten textual datasets, four multimodal benchmarks, and an end-to-end Raspberry Pi~5 retrieval-and-inference pipeline.

\end{itemize}

\paragraph{Organization.}
Section~\ref{sec:rw} reviews related work and clarifies the relationship to our preliminary conference study. Section~\ref{sec:method} presents CoRA's textual and multimodal formulations and analyzes its low-rank retrieval basis. Section~\ref{sec:hardware} derives the streaming index-construction procedure and discusses hardware compatibility. Section~\ref{sec:exp} reports the experimental setup, results, and analyses. Section~\ref{sec:conclusion} concludes the article and discusses its limitations.

\section{Related Work}\label{sec:rw}

\subsection{Exemplar Retrieval for Textual In-Context Learning}

Textual exemplar retrieval builds on lexical and dense relevance matching. BM25~\cite{robertson1995okapi} ranks candidates according to lexical overlap, whereas pretrained-language-model-based dense retrieval represents queries and candidates in a latent semantic space. \citet{zhao2024dense} organize dense retrieval methods along four dimensions: architecture, training, indexing, and integration with other retrieval signals. DHR combines dense lexical and semantic representations for more efficient hybrid retrieval~\cite{lin2023dense}. These methods provide relevance matching, but do not directly model whether a retrieved example is useful as an ICL demonstration.

ICL-specific methods adapt retrieval to demonstration selection. KATE~\cite{liu2022makes} retrieves nearest neighbors using sentence representations, showing that semantically related demonstrations outperform random selection. Skill-KNN~\cite{an2023skill} prompts a language model to describe skills required by candidate inputs and test queries, then retrieves demonstrations using these descriptions. This reduces sensitivity to task-irrelevant surface features without fine-tuning, but requires an additional generation stage and does not use candidate-memory outputs.

Learning-based methods obtain stronger task signals from downstream model behavior. EPR~\cite{rubin2022learning} uses the language-model likelihood of the gold output to label candidate prompts as positive or negative, and trains a dense retriever on the resulting supervision. CEIL~\cite{ye2023compositional} models demonstration selection as a subset problem using a conditional determinantal point process, thereby accounting for both query relevance and interactions among selected examples. UDR~\cite{li2023unified} unifies supervision from heterogeneous tasks through multi-task list-wise ranking and iterative candidate mining. These methods capture richer notions of demonstration utility than fixed input similarity, but require language-model-derived supervision, retriever training, or both, thereby increasing offline data-construction and optimization costs.

\subsection{Exemplar Retrieval for Multimodal In-Context Learning}

Multimodal ICL requires a VLM to process prompts containing interleaved images, textual inputs, and outputs. Flamingo~\cite{alayrac2022flamingo} established the effectiveness of interleaved multimodal demonstrations across vision-language tasks. \citet{doveh2024towards} introduce a curriculum-based ICL instruction-tuning approach that improves a VLM's ability to follow multimodal demonstrations. MMICL~\cite{zhaommicl} develops a multimodal context scheme and instruction-tuning dataset for prompts containing multiple images. These methods primarily improve how the downstream model consumes multimodal demonstrations rather than how a standalone retriever selects them.

Retrieval methods determine which multimodal demonstrations to include in the prompt. MMICES~\cite{chen2025can} first filters candidate demonstrations according to visual similarity and then ranks them using language similarity. TACO~\cite{li2025taco} interprets each demonstration as a local task mapping and configures demonstration sequences that support a coherent global mapping for the query. \citet{chen2025provoking} formulate selection as an exploration-exploitation process, using VLM-derived rewards and policy-gradient optimization to learn demonstration combinations. These approaches demonstrate the value of cross-modal and set-level selection, but depend on fixed fusion strategies, learned selection modules, or repeated downstream-model interaction.

Privacy is an additional concern when visual and textual retrieval data contain user-specific information. \citet{zhang2023privacy} proactively transform cross-modal data into adversarially protected versions before release, causing malicious retrieval models to produce erroneous retrieval results on the protected data. Our setting considers privacy at a different stage of the retrieval pipeline: paired memories and their representations remain on device, and retrieval is performed locally without transferring them to an external service. These approaches respectively protect data intended for release and reduce data exposure through local storage and computation.

\subsection{On-Device Retrieval and Index Efficiency}

Resource-constrained retrieval must account for representation computation, index storage and maintenance, and query-time ranking. Fast-Forward indexes~\cite{leonhardt2024efficient} combine precomputed representations, lightweight encoders, lexical-semantic score interpolation, and index-reduction techniques to support low-latency neural reranking without GPU acceleration. Although developed for document ranking, this separation between offline representation construction and online scoring is also relevant to on-device exemplar retrieval.

On-device retrieval has also been studied in retrieval-augmented generation, where stored representations and generation contexts compete for limited memory. ECG models~\cite{killingback2026unified} use shared representations for document retrieval, context compression, and response generation. Our setting instead retrieves input-output ICL demonstrations and uses stored outputs to condition a frozen input representation space. It must therefore consider task-conditioned index construction in addition to query-time nearest-neighbor search.

\subsection{Relationship to Our Preliminary Conference Work}

Our preliminary conference work~\cite{liu2024unraveling} investigated which similarity signals are acquired by learning-based demonstration retrievers. It formulated and empirically examined two hypotheses:
\begin{itemize}
    \item \textbf{H1:} a learning-based retriever adaptively integrates task-agnostic input similarities encoded at different representational levels;
    \item \textbf{H2:} beyond input similarity, a learning-based retriever implicitly favors demonstrations whose outputs are similar to the unknown output associated with the test query.
\end{itemize}
H1 was examined by comparing the retrieval behavior of different pretrained encoder layers and measuring their CKA similarity to the final representations of pretrained retrievers. H2 was examined by comparing the output similarities of positive and negative training examples and of demonstrations selected by learned and task-agnostic retrievers.

The conference work introduced MLSM and TTF as separate realizations of these hypotheses. MLSM addresses H1 while keeping the underlying BERT encoder frozen. It first computes pairwise CKA scores among encoder layers on an unlabeled subset of the demonstration set, clusters the resulting layer-similarity profiles with $k$-means, and selects the central layer from each cluster as a representative similarity expert. For a query, each selected layer produces a temperature-scaled distribution over a sampled set of candidate demonstrations according to cosine similarity. MLSM then aggregates these distributions with normalized layer weights, which are optimized on a sampled validation set by minimizing an agreement loss between the aggregated and layer-wise distributions, with early stopping. The weights may be estimated for an individual query or updated using a batch of queries. MLSM therefore combines complementary input-side similarities, but neither candidate outputs nor target-task labels enter its retrieval objective.

TTF addresses H2 through task-specific retriever adaptation. It fine-tunes a retriever jointly with an auxiliary prediction module using labeled examples from the target demonstration set. For classification tasks, the auxiliary module is a classification head. The retrieval representation is trained to support label prediction, such that examples associated with the same output are encouraged to be close to the corresponding label-specific classifier representation. For generation tasks, TTF adopts an encoder--decoder architecture, where the decoder supplies sequence-generation supervision and average-pooled final-layer encoder representations are used for retrieval. This supervision jointly shapes the predictive module and the retrieval representation. Thus, TTF injects task-specific input-output relations into the retriever through gradient-based parameter fine-tuning.

Although MLSM and TTF validated the complementary roles of multi-level input similarity and output-related task information, they treated these signals separately. MLSM requires query- or batch-level weight optimization and remains output-agnostic, whereas TTF obtains task specificity through labeled fine-tuning of the retriever. The conference study was also limited to textual retrieval and did not address compact index construction, multimodal conditioning, or end-to-end deployment under edge resource constraints.

\section{Methodology}\label{sec:method}
\subsection{Problem Setup and Preliminaries}
\paragraph{On-device Exemplar Retrieval for ICL}
We consider exemplar retrieval as a local pre-inference operation performed before invoking the downstream model in on-device ICL. Given a test input \(\mathbf{x}^t\) and a local candidate pool \(\mathcal{D} = \{(\mathbf{x}_i^c, \mathbf{y}_i^c)\}_{i=1}^n\) of input-output pairs, the goal is to select a subset \(\mathcal{S}(\mathbf{x}^t) \subseteq \mathcal{D}\) with \(|\mathcal{S}(\mathbf{x}^t)| = k \ll n\) that best supports the downstream large model \(\mathcal{M}\) for the target query. The retrieved exemplars are then combined with \(\mathbf{x}^t\) to form the prompt \(p(\mathbf{x}^t,\mathcal{S})\), which is provided to the downstream model to produce \(\mathbf{y}^t = \mathcal{M}(p(\mathbf{x}^t,\mathcal{S}))\).

The retriever has access to a frozen pretrained encoder \(f\) on device. Here, \(\mathbf{x}\) denotes either a textual input or, more generally, a multimodal instance. We use \emph{gradient-free} to mean that CoRA performs no gradient-based retriever optimization. The method nevertheless conducts task-conditioned adaptation by constructing a retrieval basis in closed-form from paired candidate memory. It does not backpropagate through \(\mathcal{M}\) or invoke the target model during index construction or exemplar selection. All retrieval computations run locally under the available memory and latency budgets. We first present the textual formulation and subsequently extend it to multimodal retrieval.

\paragraph{Layerwise Representations as Retrieval Anchors.}
Probing studies on pretrained transformers have shown that linguistic information is distributed hierarchically across layers, with lower layers capturing surface-level cues, middle layers encoding syntactic regularities, and upper layers emphasizing higher-level semantics~\cite{jawahar2019does, ma2019universal}. Recent analyses of dense retrieval further suggest that retrieval effectiveness is often associated with a task-dependent combination of layerwise similarities, rather than a universal reliance on the final-layer embedding alone~\cite{liu2024unraveling}. Taken together, these observations indicate that intermediate representations can provide complementary retrieval cues, and that the most informative layers may vary across tasks, datasets, and application domains.

Motivated by this evidence, we treat layerwise encoder states as candidate retrieval anchors rather than assuming that the final-layer representation is universally sufficient. For a textual sequence \(\mathbf{s}\), let \(\mathcal{I}(\mathbf{s})\) denote the set of its non-padding token indices. We define its layer-\(\ell\) representation as
\begin{equation}
\mathbf{h}_{\ell}(\mathbf{s})
=
\frac{1}{|\mathcal{I}(\mathbf{s})|}
\sum_{q \in \mathcal{I}(\mathbf{s})}
\mathbf{h}_{\ell,q}(\mathbf{s}),
\end{equation}
where \(\mathbf{h}_{\ell,q}(\mathbf{s})\) is the hidden state of token \(q\) at layer \(\ell\). The same frozen encoder \(f\) is used to encode candidate inputs, candidate outputs, and test inputs. CoRA uses these layerwise representations to construct an output-conditioned retrieval basis without updating the encoder parameters.

\begin{tcolorbox}[width=\linewidth, colback=white!95!black]
\noindent \textbf{Method Desiderata.} 
The above setup and preliminaries suggest three desiderata for our retrieval alignment: \(i\)) task-aware retrieval under frozen-encoder constraints, \(ii\)) effective use of complementary signals across layers, and \(iii\)) minimal additional overhead for on-device deployment. The next subsection presents CoRA, a gradient-free multi-layer retrieval alignment method designed to satisfy these desiderata.
\end{tcolorbox}

\subsection{CoRA: Task-Conditioned Multi-Layer Retrieval Alignment} 
\begin{figure}
    \centering
    \includegraphics[width=\linewidth]{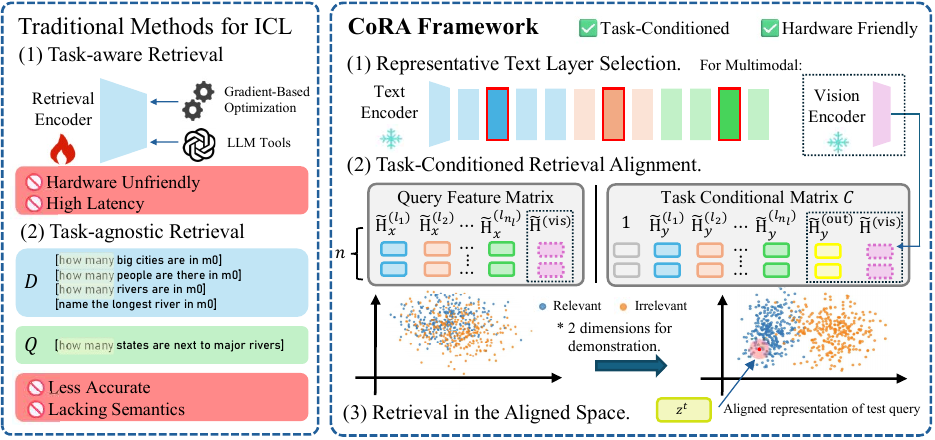}
    \caption{\textbf{Left:} Shortcomings of existing retrieval methods. \textbf{Right:} Overview of CoRA framework, which consists of three main stages. By projecting representations into a task-conditioned subspace without gradient-based parameter updates, it enables effective textual and multimodal retrieval on edge devices.}
    \Description{Research gap and the main schematic diagram.}
    \label{fig:cora}
\end{figure}
\paragraph{Framework Overview.}
As shown in Figure~\ref{fig:cora}, CoRA converts a frozen pretrained encoder into a task-conditioned multi-layer retriever through three stages. First, it applies representative-layer selection to retain complementary input-side features while reducing inter-layer redundancy. Second, it uses the paired candidate outputs to construct a conditioning matrix and fits the selected input representations to this matrix through closed-form ridge regression. A low-rank basis is then estimated from the resulting fitted representation matrix. Third, candidate and query inputs are mapped through this basis and ranked by similarity. Candidate outputs are required only during index construction and an unseen query is encoded using its input alone. The formulation extends to multimodal retrieval, and all computations reduce to lightweight closed-form linear algebra suitable for on-device deployment.

\paragraph{Representative Layer Selection.}
CoRA adopts the representative-layer construction introduced by MLSM in our preliminary conference work~\cite{liu2024unraveling}. This inherited component uses inter-layer similarity to retain complementary input-side representations while reducing the redundancy caused by residual propagation~\cite{dalvi2020analyzing,liu2024unraveling,min2025docs}. CoRA builds on this construction by introducing the output-conditioned alignment described in the following paragraph. Let the frozen retriever encoder \(f\) have \(L\) transformer layers, indexed by \(\ell \in \{1,\dots,L\}\). We draw a small calibration subset $\mathcal{D}_s=\{(\mathbf{x}_i^c,\mathbf{y}_i^c)\}_{i=1}^{n_s}\subset \mathcal{D},$ and use only the inputs \(\{\mathbf{x}_i^c\}_{i=1}^{n_s}\) in this stage to estimate inter-layer redundancy. For each sample \(\mathbf{x}_i^c\) and each layer \(\ell\), we instantiate the layerwise token-average representation $\mathbf{h}_\ell(\mathbf{x}_i^c) \in \mathbb{R}^{d}.$ Stacking these vectors across the subset gives
\begin{equation}
\mathbf{H}^{(\ell)} = [\mathbf{h}_\ell(\mathbf{x}_1^c);\dots;\mathbf{h}_\ell(\mathbf{x}_{n_s}^c)] \in \mathbb{R}^{n_s \times d},
\end{equation}
which summarizes the empirical behavior of layer \(\ell\).

We quantify the similarity between layers using linear Centered Kernel Alignment (CKA)~\cite{kornblith2019similarity}. Specifically, we form a similarity matrix \(\mathbf{S} \in \mathbb{R}^{L \times L}\) with entries
\begin{equation}
\mathbf{S}_{\ell,\ell'} = \mathrm{CKA}(\mathbf{H}^{(\ell)}, \mathbf{H}^{(\ell')}).
\end{equation}
Each row \(\mathbf{S}_{\ell,\cdot}\) represents the similarity profile of layer \(\ell\) with respect to all other layers. Because adjacent layers are often highly similar, selecting layers directly from raw pairwise similarities would favor neighboring and thus redundant layers. To encourage diversity, we treat each row vector \(\mathbf{S}_{\ell,\cdot}\) as a layer feature and cluster \(\{\mathbf{S}_{\ell,\cdot}\}_{\ell=1}^{L}\) into \(n_l\) groups using \(k\)-means. For each cluster \(\mathcal{C}_j\), we then choose as its representative the layer with the largest average similarity to the layers in the same cluster. This yields the representative layer set $\mathcal{L}_{\mathrm{key}} = \{\ell_j\}_{j=1}^{n_l},$
which suppresses redundancy while retaining diverse layerwise retrieval cues for the subsequent task-conditioned alignment stage.

\paragraph{Task-Conditioned Retrieval Alignment.}
Representative layer selection suppresses redundancy, but the resulting features remain task-agnostic with respect to the downstream objective. They reflect the geometry of the frozen encoder rather than the input-output relation induced by the task for which exemplars are retrieved. In on-device ICL, learning this relation with gradient updates, downstream heads, or target-model feedback is infeasible, since the downstream model \(\mathcal{M}\) is invoked only after retrieval and all intermediate computation must remain local. We therefore seek a closed-form alignment mechanism that uses only the local exemplar pairs in \(\mathcal{D}\) while preserving the frozen nature of the encoder. Our key intuition is to identify the components of the input representations that are systematically associated with the output-side structure of the candidate pairs. This relation provides an output-conditioned retrieval space without requiring any update to the encoder parameters.

For each \(\ell_j \in \mathcal{L}_{\mathrm{key}}\), we encode both the candidate inputs and their textual outputs using the same frozen encoder \(f\), apply mean pooling over non-padding token representations, and form
\(\mathbf{H}_x^{(\ell_j)},\,\mathbf{H}_y^{(\ell_j)} \in \mathbb{R}^{n \times d}\),
where the \(i\)-th rows of \(\mathbf{H}_x^{(\ell_j)}\) and \(\mathbf{H}_y^{(\ell_j)}\) are the layer-\(\ell_j\) features of \(\mathbf{x}_i^c\) and \(\mathbf{y}_i^c\), respectively. We standardize each feature dimension via Z-score normalization within each layer and obtain \(\tilde{\mathbf{H}}_x^{(\ell_j)},\, \tilde{\mathbf{H}}_y^{(\ell_j)}.\) To aggregate task signals from exemplar outputs across layers, we construct the conditioning matrix
\begin{equation}
\mathbf{C}
=
\big[\, \mathbf{1} \;\big|\; \tilde{\mathbf{H}}_y^{(\ell_1)} \;\big|\; \dots \;\big|\; \tilde{\mathbf{H}}_y^{(\ell_{n_l})} \,\big]
\in \mathbb{R}^{n \times p},
\end{equation}
where \(\mathbf{1} \in \mathbb{R}^{n}\) is the all-ones column vector. For notational convenience, let
\(d_{\mathrm{cat}}=n_l d\) denote the concatenated input-feature dimension and
\(p=1+d_{\mathrm{cat}}\) the column dimension of \(\mathbf{C}\).

For each selected layer \(\ell_j\), we obtain the component of its input representations fitted from the output-derived conditioning matrix through ridge regression:
\begin{equation}\label{eq_fitting}
\hat{\mathbf{H}}_x^{(\ell_j)} = \mathbf{P}_{\mathbf{C}}\tilde{\mathbf{H}}_x^{(\ell_j)}, \qquad \mathbf{P}_{\mathbf{C}} = \mathbf{C} \left(\mathbf{C}^{\top}\mathbf{C} + \lambda\mathbf{I} \right)^{-1} \mathbf{C}^{\top},
\end{equation}
where \(\lambda>0\) is the ridge regularization coefficient. Here,
\(\mathbf{P}_{\mathbf{C}}\) is the regularized fitting operator
induced by \(\mathbf{C}\). Accordingly, the fitted matrix
\(\hat{\mathbf{H}}_x^{(\ell_j)}\) captures the component of the
layerwise input representations explained by the output-derived
conditioning matrix, while leaving the encoder parameters unchanged.

We then consolidate the fitted components and the corresponding standardized input representations from all representative layers by concatenation
\begin{equation}
\hat{\mathbf{H}}_x
=
\big[\, \hat{\mathbf{H}}_x^{(\ell_1)} \;\big|\; \dots \;\big|\; \hat{\mathbf{H}}_x^{(\ell_{n_l})} \,\big]
\in \mathbb{R}^{n \times d_{\mathrm{cat}}},
\qquad
\tilde{\mathbf{H}}_x
=
\big[\, \tilde{\mathbf{H}}_x^{(\ell_1)}
\;\big|\; \dots \;\big|\;
\tilde{\mathbf{H}}_x^{(\ell_{n_l})} \,\big]
\in \mathbb{R}^{n \times d_{\mathrm{cat}}}.
\end{equation}
Finally, to obtain a compact retrieval space suitable for on-device deployment, we perform a spectral decomposition \(\hat{\mathbf{H}}_x = \mathbf{U}\mathbf{\Sigma}\mathbf{V}^{\top},\)
and retain the top-\(r\) right singular vectors \(\mathbf{V}_{1:r} \in \mathbb{R}^{d_{\mathrm{cat}} \times r}\). The resulting aligned representation is
\begin{equation}\label{eq_decomposition}
\mathbf{Z} = \tilde{\mathbf{H}}_x \mathbf{V}_{1:r} \in \mathbb{R}^{n \times r},
\end{equation}
whose columns span a compact task-conditioned subspace for retrieval. The matrix \(\mathbf{Z}\) forms the compact candidate index used for retrieval. Because \(\mathbf{V}_{1:r}\) is estimated from \(\hat{\mathbf{H}}_x\), the retained feature directions are determined by the dominant structure of the output-conditioned fit, while candidate and query representations are projected from their standardized
input-side features. This also remains fully gradient-free and lightweight enough for on-device deployment.

\paragraph{Retrieval in the Aligned Space.}
CoRA performs retrieval in the learned aligned space. Let \(\bar{\mathbf{Z}}=\mathrm{rownorm}(\mathbf{Z})\) be its row-wise \(\ell_2\)-normalized form. For a test query \(\mathbf{x}^t\), we extract layerwise token-average representations from representative layers \(\mathcal{L}_{\mathrm{key}}\), standardize them using the alignment statistics, and obtain its aligned embedding
\begin{equation}
\mathbf{z}^t
=
\mathrm{norm}\!\left(
\big[
\tilde{\mathbf{h}}_{\ell_1}(\mathbf{x}^t) \;|\; \dots \;|\; \tilde{\mathbf{h}}_{\ell_{n_l}}(\mathbf{x}^t)
\big]
\mathbf{V}_{1:r}
\right)
\in \mathbb{R}^{r},
\end{equation}
where \(\mathrm{norm}(\cdot)\) denotes \(\ell_2\) normalization. Although \(\mathbf{V}_{1:r}\) is learned through target-conditioned alignment, applying it to a test query requires only its input features, since the task signal is already encoded in the basis. We then compute the cosine similarity between \(\mathbf{z}^t\) and each row \(\bar{\mathbf{Z}}_{i,:}\), and select the top-\(k\) exemplars to form the ICL prompt. Because \(\bar{\mathbf{Z}}\) can be precomputed offline, query-time retrieval only involves a lightweight projection and nearest-neighbor search in \(\mathbb{R}^{r}\), making it compatible with standard ANN indexing and the pre-inference on-device setting.

\paragraph{Extension to Multimodal Retrieval.}
Prior work on multimodal ICL suggests that textual content often plays a stronger role than visual content~\cite{baldassini2024makes,chen2025can}. We therefore extend CoRA to multimodal retrieval in a text-centric manner, denoted as \textbf{CoRA-M}. Representative layer selection remains unchanged and is applied only to the textual transformer layers, yielding the same set \(\mathcal{L}_{\mathrm{key}}\). Visual information is introduced through the final pooled representation produced by the frozen vision encoder and multimodal projector.

For a multimodal exemplar \(\mathbf{x}_i^c\), let \(\mathbf{v}(\mathbf{x}_i^c) \in \mathbb{R}^{d}\) denote its projected visual embedding in the shared semantic space. Stacking these embeddings across \(\mathcal{D}\) gives
\begin{equation}
\mathbf{H}_x^{(\mathrm{vis})}
=
[\mathbf{v}(\mathbf{x}_1^c);\dots;\mathbf{v}(\mathbf{x}_n^c)]
\in \mathbb{R}^{n \times d},
\end{equation}
and let \(\tilde{\mathbf{H}}_x^{(\mathrm{vis})}\) be its standardized form. On the target side, unlike the text-only case, we do not concatenate the selected layerwise representations of \(\mathbf{y}_i^c\). Instead, we encode each target response with the frozen text encoder and use its mean-pooled final-layer representation
\begin{equation}
\mathbf{H}_{y}^{(\mathrm{out})}
=
[\mathbf{h}_{L}(\mathbf{y}_1^c);\dots;\mathbf{h}_{L}(\mathbf{y}_n^c)]
\in \mathbb{R}^{n \times d},
\end{equation}
with standardized form \(\tilde{\mathbf{H}}_{y}^{(\mathrm{out})}\). The reason is that, in the multimodal setting, the target response already serves as a compact semantic summary of the answer, while visual grounding is supplied separately by \(\tilde{\mathbf{H}}_x^{(\mathrm{vis})}\). Reintroducing multiple target-side text layers would therefore mainly enlarge the conditioning space and add redundant variation, rather than provide complementary task signal. We thus define
\begin{equation}
\mathbf{C}_{\mathrm{mm}}
=
\big[\, \mathbf{1} \;\big|\; \tilde{\mathbf{H}}_{y}^{(\mathrm{out})} \;\big|\; \tilde{\mathbf{H}}_x^{(\mathrm{vis})} \,\big]
\in \mathbb{R}^{n \times (1 + 2d)}.
\end{equation}
For each selected text layer \(\ell_j \in \mathcal{L}_{\mathrm{key}}\), we compute the target-conditioned component following Eq.~\ref{eq_fitting} with \(\mathbf{C}\) replaced by \(\mathbf{C}_{\mathrm{mm}}\). The multimodal aligned representation is then formed by concatenating the target-conditioned textual blocks with the visual block
\begin{equation}
\hat{\mathbf{H}}_x^{\mathrm{mm}}
=
\big[\, \hat{\mathbf{H}}_{x}^{(\ell_1)} \;\big|\; \dots \;\big|\; \hat{\mathbf{H}}_{x}^{(\ell_{n_l})} \;\big|\; \tilde{\mathbf{H}}_x^{(\mathrm{vis})} \,\big]
\in \mathbb{R}^{n \times (n_l+1)d}.
\end{equation}
We then apply the same spectral decomposition and low-rank projection as in the text-only case (referring to Eq.~\ref{eq_decomposition}). For a test query \(\mathbf{x}^t\), the retrieval embedding becomes
\begin{equation}
\mathbf{z}_{\mathrm{mm}}^t
=
\mathrm{norm}\!\left(
\big[
\tilde{\mathbf{h}}_{\ell_1}(\mathbf{x}^t) \;|\; \dots \;|\; \tilde{\mathbf{h}}_{\ell_{n_l}}(\mathbf{x}^t) \;|\; \tilde{\mathbf{v}}(\mathbf{x}^t)
\big]
\mathbf{V}_{1:r}
\right).
\end{equation}
This extension preserves CoRA's original text-centric retrieval backbone while allowing visual cues to enter the same target-conditioned subspace. The subsequent low-rank projection can therefore adapt the combined text-visual representation for retrieval while retaining a unified representation space across the two modalities.

\subsection{Analytical Characterization of the Low-Rank Basis}
\label{s33_theory}

We characterize the rank-constrained compression performed in CoRA's
basis-estimation step. Recall that
\[
\hat{\mathbf{H}}_x
=
\mathbf{P}_{\mathbf{C}}\tilde{\mathbf{H}}_x
\in \mathbb{R}^{n\times d_{\mathrm{cat}}}
\]
is the representation matrix fitted from the output-derived conditioning
matrix. Consider its full singular value decomposition
\[
\hat{\mathbf{H}}_x
=
\mathbf{U}\mathbf{\Sigma}\mathbf{V}^{\top},
\]
where \(\mathbf{V}\in\mathbb{R}^{d_{\mathrm{cat}}\times d_{\mathrm{cat}}}\)
is orthogonal and
\(\sigma_1\geq\cdots\geq\sigma_{d_{\mathrm{cat}}}\geq 0\)
are the singular values, padded with zeros when necessary. Let
\(\mathbf{V}_{1:r}\) contain the first \(r\) right singular vectors.

\begin{proposition}[Optimal Rank-\(r\) Compression of the Output-Conditioned Fit]
\label{prop:conditional_basis}
For \(1\leq r\leq d_{\mathrm{cat}}\),
\(\mathbf{V}_{1:r}\) is an optimal solution to
\begin{equation}
\max_{\substack{
\mathbf{W}\in\mathbb{R}^{d_{\mathrm{cat}}\times r}\\
\mathbf{W}^{\top}\mathbf{W}=\mathbf{I}_r}}
\left\|\hat{\mathbf{H}}_x\mathbf{W}\right\|_F^2,
\end{equation}
and the optimal value is
\begin{equation}
\left\|\hat{\mathbf{H}}_x\mathbf{V}_{1:r}\right\|_F^2
=
\sum_{i=1}^{r}\sigma_i^2.
\label{eq:conditional_energy_bound}
\end{equation}
Equivalently, \(\mathbf{V}_{1:r}\) solves
\begin{equation}
\min_{\substack{
\mathbf{W}\in\mathbb{R}^{d_{\mathrm{cat}}\times r}\\
\mathbf{W}^{\top}\mathbf{W}=\mathbf{I}_r}}
\left\|
\hat{\mathbf{H}}_x-
\hat{\mathbf{H}}_x\mathbf{W}\mathbf{W}^{\top}
\right\|_F^2,
\end{equation}
with minimum value
\(\sum_{i=r+1}^{d_{\mathrm{cat}}}\sigma_i^2\).
\end{proposition}

\begin{proof}
For any feasible \(\mathbf{W}\), define
\(\mathbf{Q}=\mathbf{V}^{\top}\mathbf{W}\). Since
\(\mathbf{V}\) is orthogonal and
\(\mathbf{W}^{\top}\mathbf{W}=\mathbf{I}_r\), we have
\(\mathbf{Q}^{\top}\mathbf{Q}=\mathbf{I}_r\). Using the SVD of
\(\hat{\mathbf{H}}_x\), the retained fitted energy is
\[
\left\|\hat{\mathbf{H}}_x\mathbf{W}\right\|_F^2 =
\left\|\mathbf{\Sigma}\mathbf{Q}\right\|_F^2 =
\sum_{i=1}^{d_{\mathrm{cat}}}
\sigma_i^2
\left\|\mathbf{q}_{i,:}\right\|_2^2,
\]
where \(\mathbf{q}_{i,:}\) is the \(i\)-th row of
\(\mathbf{Q}\). The column orthonormality of \(\mathbf{Q}\) implies
\[
0\leq\left\|\mathbf{q}_{i,:}\right\|_2^2\leq 1,
\qquad
\sum_{i=1}^{d_{\mathrm{cat}}}
\left\|\mathbf{q}_{i,:}\right\|_2^2=r.
\]
Because
\(\sigma_1^2\geq\cdots\geq\sigma_{d_{\mathrm{cat}}}^2\),
the weighted sum is maximized by assigning unit weight to the first
\(r\) singular directions and zero weight to the remaining directions.
Therefore,
\[
\left\|\hat{\mathbf{H}}_x\mathbf{W}\right\|_F^2
\leq
\sum_{i=1}^{r}\sigma_i^2,
\]
with equality attained by
\(\mathbf{W}=\mathbf{V}_{1:r}\).

For the equivalent reconstruction objective,
\(\mathbf{W}\mathbf{W}^{\top}\) is an orthogonal projector, and hence
\[
\left\|
\hat{\mathbf{H}}_x
-
\hat{\mathbf{H}}_x\mathbf{W}\mathbf{W}^{\top}
\right\|_F^2
=
\left\|\hat{\mathbf{H}}_x\right\|_F^2
-
\left\|\hat{\mathbf{H}}_x\mathbf{W}\right\|_F^2.
\]
Thus, minimizing the reconstruction error is equivalent to maximizing
the retained fitted energy. Substituting
\(\mathbf{W}=\mathbf{V}_{1:r}\) and using
\(\|\hat{\mathbf{H}}_x\|_F^2
=\sum_i\sigma_i^2\)
gives the minimum reconstruction error
\(\sum_{i=r+1}^{d_{\mathrm{cat}}}\sigma_i^2\).
\end{proof}

\noindent\textbf{Interpretation.}
Proposition~\ref{prop:conditional_basis} characterizes the compression step used to estimate CoRA's retrieval basis. The fitted matrix \(\hat{\mathbf{H}}_x\) summarizes input-side variation associated with the output-derived conditioning matrix, and its leading right singular subspace provides the most faithful \(r\)-dimensional representation of this fitted structure. The trailing spectral energy \(\sum_{i=r+1}^{d_{\mathrm{cat}}}\sigma_i^2\) is the exact approximation error under the rank constraint.

For any \(r\)-dimensional orthogonal basis \(\mathbf{W}\), define the normalized retained fitted energy as 
\begin{equation}
\rho_r(\mathbf{W}) = \frac{\left\|\hat{\mathbf{H}}_x\mathbf{W}\right\|_F^2}{\left\|\hat{\mathbf{H}}_x\right\|_F^2}.
\label{eq:retained_fitted_energy}
\end{equation}
Proposition~\ref{prop:conditional_basis} implies that CoRA's basis satisfies
\begin{equation}
\rho_r(\mathbf{V}_{1:r}) = \max_{\mathbf{W}^{\top}\mathbf{W}=\mathbf{I}_r} \rho_r(\mathbf{W}) = \frac{\sum_{i=1}^{r}\sigma_i^2}{\sum_{i=1}^{d_{\mathrm{cat}}}\sigma_i^2}.
\label{eq:optimal_retained_fitted_energy}
\end{equation}
Section~\ref{sec:abls} compares this quantity across alternative bases under the same rank budget and examines its relation to downstream retrieval performance.

\section{On-Device Realization of the Retrieval Pipeline}
\label{sec:hardware}
\subsection{On-Device Streaming Construction of the Retrieval Representation}

We show that the final retrieval representation $\mathbf{Z}$ can be computed in two sequential passes over the exemplar pool without materializing the full fitted matrix $\hat{\mathbf{H}}_x$. Collect the layerwise ridge-regression coefficients as
\begin{equation}
\mathbf{B}
=
\big[
\mathbf{B}^{(\ell_1)}
\;\big|\; \cdots \;\big|\;
\mathbf{B}^{(\ell_{n_l})}
\big]
\in \mathbb{R}^{p\times d_{\mathrm{cat}}}.
\end{equation}
The layerwise regularized regressions can then be written jointly as
\begin{equation}
  \mathbf{B}
  =
  \arg\min_{\mathbf{B}}
  \left\|
    \tilde{\mathbf{H}}_x - \mathbf{C}\mathbf{B}
  \right\|_F^2 + \lambda \|\mathbf{B}\|_F^2,
  \qquad
  \hat{\mathbf{H}}_x = \mathbf{C}\mathbf{B}.
\end{equation}

\paragraph{Pass~1: Sufficient Statistics and Global Solve.}
Partition $\mathcal{D}$ into chunks $\{\mathcal{D}_b\}_{b=1}^{m}$, and let $\mathbf{C}_b \in \mathbb{R}^{|\mathcal{D}_b| \times p}$ and $\tilde{\mathbf{H}}_{x,b} \in \mathbb{R}^{|\mathcal{D}_b| \times d_{\mathrm{cat}}}$ denote the corresponding row blocks of $\mathbf{C}$ and $\tilde{\mathbf{H}}_x$. A first streaming pass accumulates the two sufficient-statistic matrices
\begin{equation}
  \mathbf{G}
  =
  \sum_{b=1}^{m} \mathbf{C}_b^{\top}\mathbf{C}_b
  \in \mathbb{R}^{p \times p},
  \qquad
  \mathbf{T}
  =
  \sum_{b=1}^{m} \mathbf{C}_b^{\top}\tilde{\mathbf{H}}_{x,b}
  \in \mathbb{R}^{p \times d_{\mathrm{cat}}}.
\end{equation}
After the first pass, \(\mathbf{B}\) is obtained by solving
\begin{equation}
(\mathbf{G}+\lambda\mathbf{I}_p)\mathbf{B}
=
\mathbf{T}.
\end{equation}
The dimensions of this system depend on \(p\) and \(d_{\mathrm{cat}}\), but not on the number of exemplars \(n\). The low-rank basis required for Eq.~\ref{eq_decomposition} can likewise be obtained without forming $\hat{\mathbf{H}}_x$. Since $\hat{\mathbf{H}}_x = \mathbf{C}\mathbf{B}$,
\begin{equation}
  \hat{\mathbf{H}}_x^{\top}\hat{\mathbf{H}}_x
  =
  (\mathbf{C}\mathbf{B})^{\top}(\mathbf{C}\mathbf{B})
  =
  \mathbf{B}^{\top}\mathbf{G}\,\mathbf{B}
  \in \mathbb{R}^{d_{\mathrm{cat}} \times d_{\mathrm{cat}}}.
\end{equation}
The top-$r$ right singular vectors $\mathbf{V}_{1:r}$ are therefore the leading eigenvectors of $\mathbf{B}^{\top}\mathbf{G}\,\mathbf{B}$, a $d_{\mathrm{cat}} \times d_{\mathrm{cat}}$ matrix whose size is independent of $n$.

\paragraph{Pass~2: Chunkwise Index Construction.}
After obtaining \(\mathbf{V}_{1:r}\), a second sequential pass re-encodes each chunk, applies the input-side standardization statistics, and directly constructs its retrieval vectors as
\begin{equation}
\mathbf{Z}_b
=
\tilde{\mathbf{H}}_{x,b}\mathbf{V}_{1:r}
\in \mathbb{R}^{|\mathcal{D}_b| \times r}.
\end{equation}
Concatenating the chunkwise results yields
\(\mathbf{Z}\in\mathbb{R}^{n\times r}\), exactly matching the index representation defined in Eq.~\ref{eq_decomposition}.

\paragraph{Peak Memory Footprint.}
During the first pass, CoRA retains only one pair of feature blocks \(\mathbf{C}_b\) and \(\tilde{\mathbf{H}}_{x,b}\), together with the sufficient-statistic matrices \(\mathbf{G}\) and \(\mathbf{T}\). The global solve additionally produces the coefficient matrix \(\mathbf{B}\) and the \(d_{\mathrm{cat}}\times d_{\mathrm{cat}}\) matrix \(\mathbf{B}^{\top}\mathbf{G}\mathbf{B}\), from which the basis \(\mathbf{V}_{1:r}\) is obtained. During the second pass, each input-feature block is projected using \(\mathbf{V}_{1:r}\) and written incrementally to the retrieval index. Excluding the persistent storage of the resulting index, the construction-time working memory is therefore bounded by
\begin{equation}\label{eq_overhead}
\mathcal{O}\!\left(
|\mathcal{D}_b|(p+d_{\mathrm{cat}})
+p(p+d_{\mathrm{cat}})
+d_{\mathrm{cat}}^2
\right).
\end{equation}
For fixed \(|\mathcal{D}_b|\), \(p\), and \(d_{\mathrm{cat}}\), this working-memory bound is independent of the number of exemplars \(n\). The resulting retrieval index \(\mathbf{Z}\in\mathbb{R}^{n\times r}\) requires a separate \(\mathcal{O}(nr)\) storage term. The two-pass construction thus trades an additional sequential scan for working memory controlled by the chunk size and feature dimensions. Representative-layer selection further reduces this cost by keeping \(n_l\), and hence \(d_{\mathrm{cat}}\), small. Section~\ref{sec:exp_sys} empirically evaluates the resulting memory scaling as the candidate pool grows.

\paragraph{Extension to CoRA-M.}
The identical streaming construction applies to the multimodal variant CoRA-M by replacing $\mathbf{C}$ with $\mathbf{C}_{\mathrm{mm}} \in \mathbb{R}^{n \times (1+2d)}$ and $\tilde{\mathbf{H}}_x$ with the multimodal concatenation
$[\,\tilde{\mathbf{H}}_x^{(\ell_1)} \;|\; \cdots \;|\;
   \tilde{\mathbf{H}}_x^{(\ell_{n_l})} \;|\;
   \tilde{\mathbf{H}}_x^{(\mathrm{vis})}\,]
\in \mathbb{R}^{n \times (n_l+1)d}$.
This changes $d_{\mathrm{cat}}$ to $(n_l+1)d$ and $p$ to $1 + 2d$, and no other changes to the streaming procedure are required.

\subsection{Compatibility with Existing and Emerging Hardware.}
Because the entire CoRA pipeline reduces to standard dense linear algebra primitives, matrix multiplications, a linear-system solve, an eigen-decomposition, and an inner-product-based nearest-neighbor search, it can be implemented on conventional CPU-based edge platforms using highly optimized numerical libraries such as BLAS/LAPACK or their embedded variants~\cite{laug}. Looking ahead, recent work on in-memory computing (IMC) architectures has demonstrated hardware support for each of these primitives individually, including matrix inversion for linear regression~\cite{zuo2025precise}, low-rank SVD kernels~\cite{mannocci2023memory}, and inner-product-based $k$-nearest-neighbor~\cite{nie2025pick} or approximate-nearest-neighbor search~\cite{liu2025seim}. Since CoRA is composed solely of such operations, it is in principle compatible with IMC-style accelerators and could benefit from their advantages in latency and energy efficiency.

\section{Experiment}\label{sec:exp}
\subsection{Experimental Setup}
\paragraph{Tasks and Datasets.}
Following our conference study~\cite{liu2024unraveling} and prior work~\cite{ye2023compositional}, we evaluate CoRA on the same textual benchmark suite of ten publicly available datasets, comprising five classification datasets and five generation datasets across seven task categories. For each dataset, we construct a candidate pool $\mathcal{D}=\{(\mathbf{x}_i^c,\mathbf{y}_i^c)\}_{i=1}^{n}$ and a disjoint evaluation set $\mathcal{T}=\{(\mathbf{x}_j^t,\mathbf{y}_j^t)\}_{j=1}^{m}$. Given a test input $\mathbf{x}^t$, the retriever selects $\mathcal{S}(\mathbf{x}^t)\subseteq\mathcal{D}$ with $|\mathcal{S}(\mathbf{x}^t)|=k$ to construct the in-context prompt. The evaluation target $\mathbf{y}^t$ is used only to compute the downstream metric and is never accessed during retrieval or alignment. The candidate and evaluation sets contain no overlapping instances. Table~\ref{tab:datasets} summarizes the task type, split, candidate-pool size, evaluation-set size, and metric for each dataset.

\begin{table}[t]
\centering
\caption{Summary of datasets and tasks. Each dataset is divided into a local candidate pool $\mathcal{D} = \{(\mathbf{x}_i^c, \mathbf{y}_i^c)\}_{i=1}^n$ and an evaluation query set $\mathcal{T} = \{\mathbf{x}_j^t\}_{j=1}^m$. Given a test input $\mathbf{x}^t$, the retriever selects $\mathcal{S}(\mathbf{x}^t)\subseteq\mathcal{D}$ with $|\mathcal{S}(\mathbf{x}^t)|=k$ to construct the in-context prompt.}
\label{tab:datasets}
\begin{tabular}{lccc}
\toprule
\multicolumn{4}{c}{\textbf{Classification Tasks} \quad Metric: Accuracy (Acc.) $\uparrow$} \\
\midrule
Dataset & Task Type & Corpus Size & Test Queries \\
\midrule
SST-5~\cite{socher2013recursive} & Sentiment Classification & 8534 & 1101 \\
MRPC~\cite{dolan2004unsupervised} & Paraphrase Detection & 3668 & 408 \\
QNLI~\cite{wang-etal-2018-glue} & Natural Language Inference & 104707 & 5463 \\
CMSQA~\cite{talmor2019commonsenseqa} & Commonsense Reasoning & 9740 & 1221 \\
HellaSwag~\cite{zellers2019hellaswag} & Commonsense Reasoning & 52611 & 20006 \\
\midrule
\multicolumn{4}{c}{\textbf{Generation Tasks} \quad Metric: Exact Match (EM) $\uparrow$} \\
Dataset & Task Type & Corpus Size & Test Queries \\
\midrule
WebQs~\cite{berant2013semantic} & Open-Domain QA & 3778 & 2032 \\
GeoQuery~\cite{zelle1996learning} & Code Generation & 404 & 208 \\
Nl2Bash~\cite{lin2018nl2bash} & Code Generation & 7441 & 609 \\
MTOP~\cite{li2021mtop} & Semantic Parsing & 15564 & 2235 \\
SMCalFlow~\cite{andreas2020task} & Semantic Parsing & 102491 & 14751 \\
\bottomrule
\end{tabular}
\end{table}

\paragraph{Baselines.}
We compare CoRA with representative retrieval methods that are feasible for local or pre-inference deployment under limited on-device computational budgets. The baselines include: \textbf{Random}, which samples $k$ exemplars from the candidate pool; \textbf{Top-$k$ BM25}, which retrieves exemplars using BM25~\cite{robertson1995okapi}; \textbf{Top-$k$ SBERT}, which replaces BERT with sentence-BERT~\cite{reimers2019sentence}; \textbf{Top-$k$ BERT}, which ranks exemplars by the cosine similarity between their average final-layer BERT embeddings and the query~\cite{devlin2019bert}; \textbf{DPP-BERT}, which applies MAP inference with a determinantal point process on Top-$k$ BERT candidates for diverse subset selection~\cite{chen2018fast}; \textbf{Top-$k$ Qwen3-Embedding}, which ranks exemplars using cosine similarity between embeddings produced by the dedicated embedding model Qwen3-Embedding-0.6B~\cite{qwen3embedding}; and our previously proposed lightweight gradient-based retriever, \textbf{MLSM}~\cite{liu2024unraveling}, which combines layerwise BERT similarities by optimizing their agreement. We additionally compare CoRA with the training-based retrievers EPR~\cite{rubin2022learning}, CEIL~\cite{ye2023compositional}, and TTF~\cite{liu2024unraveling} in Section~\ref{sec:exp_sys}, where their task-adaptation cost is included in the efficiency comparison.

\paragraph{Implementation Details.}
Following the textual retrieval configuration established in our conference study~\cite{liu2024unraveling}, we use BERT~\cite{devlin2019bert} as the frozen retrieval encoder, select $n_l{=}3$ representative layers from a calibration subset of $n_s{=}1000$ candidate inputs, and retrieve $k{=}20$ exemplars for each query. We keep these shared settings fixed to enable direct comparison with MLSM and focus the evaluation on the task-conditioned alignment introduced by CoRA. We set the ridge regularization coefficient to $\lambda{=}10^{-6}$, and retain a conditional subspace of dimension $r{=}256$. Input and output representations are standardized independently for each selected layer using statistics computed from the candidate pool. For generation datasets, the reference response associated with each candidate is used to construct its output representation; evaluation responses are not used by the retriever. The downstream language models are instruction-tuned \textit{Llama-3.2-1B}~\cite{grattafiori2024llama}, 4-bit-quantized \textit{MobileLLM-Pro}~\cite{huber2025mobilellm}, and \textit{Qwen3.5-2B}~\cite{qwen3.5}. All retrieval methods use the same candidate pools, query sets, number of retrieved exemplars, prompt templates, exemplar ordering rule, and downstream decoding configuration. We report accuracy for classification and exact match for generation.

\paragraph{Multimodal Extension Setup.}
We evaluate CoRA-M on four multimodal exemplar-retrieval benchmarks: VQAv2~\cite{goyal2017making}, OKVQA~\cite{schwenk2022okvqa}, VizWiz~\cite{gurari2018vizwiz}, and MSCOCO~\cite{chen2015microsoft}. The first three evaluate visual question answering, whereas MSCOCO evaluates image captioning. For each benchmark, the candidate pool and evaluation set are disjoint, and evaluation targets are withheld from all retrieval procedures. CoRA-M uses CLIP~\cite{radford2021learning} to obtain the textual and visual representations.

Following established multimodal ICL protocols~\cite{baldassini2024makes,chen2025can}, we compare CoRA-M with \textbf{Random}, \textbf{RICES-image}, \textbf{RICES-text}~\cite{alayrac2022flamingo}, and \textbf{MMICES}~\cite{chen2025can}. The downstream VLMs are \textit{OpenFlamingo-3B}~\cite{awadalla2023openflamingo} and \textit{Qwen3.5-2B}~\cite{qwen3.5}. All methods use the same candidate pools, number and order of demonstrations, prompt templates, and decoding configuration. We report the standard VQA accuracy for VQAv2, OKVQA, and VizWiz, and CIDEr for MSCOCO.

\subsection{Retrieval Effectiveness}
\begin{table}[t]
\caption{Main results on five textual classification benchmarks across three language models. ``\ding{61}'' indicates that gradient computation is required. Best score for each task is boldfaced, and Avg. denotes the average score across all five tasks.}
\label{tab:main-classification}
\centering
\begin{tabular}{l ccccc >{\columncolor{blue!6}}c}
\toprule
Classification & SST-5 & MRPC & QNLI & CMSQA & SWAG & \textbf{Avg.} \\
\midrule
\multicolumn{7}{c}{Llama-3.2-1B} \\
\midrule
Random                & 0.3370 & 0.6912 & 0.5931 & 0.6544 & 0.6563 & 0.5864 \\
Top-$k$ BM25          & 0.3824 & 0.6961 & 0.6434 & 0.6495 & 0.6744 & 0.6091 \\
Top-$k$ SBERT         & 0.3787 & 0.7328 & 0.6354 & 0.6609 & 0.6831 & 0.6182 \\
Top-$k$ BERT          & 0.3878 & 0.7181 & 0.6577 & 0.6749 & 0.6828 & 0.6243 \\
DPP-BERT              & 0.3733 & 0.6961 & 0.6755 & 0.6740 & 0.6868 & 0.6211 \\
Top-$k$ Qwen3-Emb.    & 0.3987 & 0.7206 & 0.6357 & 0.6667 & \textbf{0.6902} & 0.6224 \\
MLSM\textsuperscript{\ding{61}} & 0.4024 & 0.7230 & 0.6822 & 0.6593 & 0.6850 & 0.6304 \\
\rowcolor{pink!28}
CoRA (ours)           & \textbf{0.4134{\scriptsize$\pm$0.02}} & \textbf{0.7333{\scriptsize$\pm$0.05}} & \textbf{0.7317{\scriptsize$\pm$0.02}} & \textbf{0.6753{\scriptsize$\pm$0.02}} & 0.6892{\scriptsize$\pm$0.01} & \cellcolor{pink!28}\textbf{0.6486{\scriptsize$\pm$0.02}} \\
\midrule
\multicolumn{7}{c}{MobileLLM-Pro} \\
\midrule
Random                & 0.4187 & 0.5662 & 0.5872 & 0.6814 & 0.6471 & 0.5801 \\
Top-$k$ BM25          & 0.4668 & 0.6299 & 0.6672 & 0.6585 & 0.6614 & 0.6168 \\
Top-$k$ SBERT         & 0.4469 & 0.6078 & 0.6194 & 0.6650 & 0.6678 & 0.6014 \\
Top-$k$ BERT          & 0.4432 & 0.6275 & 0.6498 & 0.6683 & 0.6681 & 0.6114 \\
DPP-BERT              & 0.4587 & 0.6275 & 0.6863 & 0.6699 & 0.6724 & 0.6229 \\
Top-$k$ Qwen3-Emb.    & 0.4614 & 0.6348 & 0.6273 & 0.6773 & \textbf{0.6735} & 0.6149 \\
MLSM\textsuperscript{\ding{61}} & 0.4614 & 0.6176 & 0.6365 & 0.6773 & 0.6688 & 0.6123 \\
\rowcolor{pink!28}
CoRA (ours)           & \textbf{0.4834{\scriptsize$\pm$0.01}} & \textbf{0.6495{\scriptsize$\pm$0.02}} & \textbf{0.7272{\scriptsize$\pm$0.01}} & \textbf{0.6825{\scriptsize$\pm$0.02}} & 0.6725{\scriptsize$\pm$0.01} & \cellcolor{pink!28}\textbf{0.6428{\scriptsize$\pm$0.01}} \\
\midrule
\multicolumn{7}{c}{Qwen3.5-2B} \\
\midrule
Random                & 0.3524 & 0.6912 & 0.7203 & 0.6822 & 0.6590 & 0.6210 \\
Top-$k$ BM25          & 0.4205 & 0.7328 & 0.7344 & 0.6257 & 0.6790 & 0.6385 \\
Top-$k$ SBERT         & 0.4051 & 0.7132 & 0.7434 & 0.6454 & 0.6888 & 0.6392 \\
Top-$k$ BERT          & 0.4133 & 0.7598 & 0.7617 & 0.6683 & 0.6857 & 0.6577 \\
DPP-BERT              & 0.4196 & 0.7157 & 0.7545 & 0.6585 & 0.6920 & 0.6481 \\
Top-$k$ Qwen3-Emb.    & 0.4233 & 0.7279 & 0.7463 & 0.6626 & 0.6954 & 0.6511 \\
MLSM\textsuperscript{\ding{61}} & 0.4033 & 0.7451 & 0.7706 & 0.6839 & 0.6873 & 0.6580 \\
\rowcolor{pink!28}
CoRA (ours)           & \textbf{0.4352{\scriptsize$\pm$0.02}} & \textbf{0.7774{\scriptsize$\pm$0.04}} & \textbf{0.8035{\scriptsize$\pm$0.03}} & \textbf{0.6924{\scriptsize$\pm$0.02}} & \textbf{0.7019{\scriptsize$\pm$0.01}} & \cellcolor{pink!28}\textbf{0.6821{\scriptsize$\pm$0.03}} \\
\bottomrule
\end{tabular}
\end{table}

\begin{table}[t]
\caption{Main results on five textual generation benchmarks across three language models. ``\ding{61}'' indicates that gradient computation is required. Best score for each task is boldfaced, and Avg. denotes the average score across all five tasks.}
\label{tab:main-generation}
\centering
\begin{tabular}{l ccccc >{\columncolor{blue!6}}c}
\toprule
Generation & WebQs & GeoQ. & NL2B. & MTOP & SMCal. & \textbf{Avg.} \\
\midrule
\multicolumn{7}{c}{Llama-3.2-1B} \\
\midrule
Random                & 0.1654 & 0.1536 & 0.1396 & 0.0515 & 0.0623 & 0.1145 \\
Top-$k$ BM25          & 0.3135 & 0.4357 & 0.3346 & 0.5074 & 0.4274 & 0.4037 \\
Top-$k$ SBERT         & 0.3209 & 0.4679 & 0.3280 & 0.4559 & 0.4154 & 0.3976 \\
Top-$k$ BERT          & 0.3204 & 0.5000 & 0.3354 & 0.5302 & 0.4551 & 0.4282 \\
DPP-BERT              & 0.3194 & 0.5214 & \textbf{0.3481} & 0.5177 & 0.4336 & 0.4280 \\
Top-$k$ Qwen3-Emb.    & 0.3140 & 0.4929 & 0.3330 & 0.4537 & 0.4196 & 0.4026 \\
MLSM\textsuperscript{\ding{61}} & 0.3243 & 0.4571 & 0.3194 & 0.5289 & 0.4747 & 0.4209 \\
\rowcolor{pink!28}
CoRA (ours)           & \textbf{0.3334{\scriptsize$\pm$0.03}} & \textbf{0.5714{\scriptsize$\pm$0.03}} & 0.3248{\scriptsize$\pm$0.04} & \textbf{0.5349{\scriptsize$\pm$0.05}} & \textbf{0.4854{\scriptsize$\pm$0.01}} & \cellcolor{pink!28}\textbf{0.4500{\scriptsize$\pm$0.04}} \\
\midrule
\multicolumn{7}{c}{MobileLLM-Pro} \\
\midrule
Random                & 0.1191 & 0.1321 & 0.1445 & 0.0157 & 0.0140 & 0.0851 \\
Top-$k$ BM25          & 0.2180 & 0.5071 & 0.2697 & 0.4179 & 0.3836 & 0.3593 \\
Top-$k$ SBERT         & 0.2116 & 0.5607 & \textbf{0.3498} & 0.3638 & 0.3796 & 0.3731 \\
Top-$k$ BERT          & 0.2180 & 0.5286 & 0.2866 & 0.4327 & 0.3992 & 0.3730 \\
DPP-BERT              & 0.2028 & 0.4071 & 0.2972 & 0.4358 & 0.3884 & 0.3463 \\
Top-$k$ Qwen3-Emb.    & 0.2195 & 0.5357 & 0.2882 & 0.3570 & 0.3847 & 0.3570 \\
MLSM\textsuperscript{\ding{61}} & 0.2190 & 0.5321 & 0.3403 & 0.4519 & 0.4164 & 0.3919 \\
\rowcolor{pink!28}
CoRA (ours)           & \textbf{0.2195{\scriptsize$\pm$0.03}} & \textbf{0.5751{\scriptsize$\pm$0.02}} & 0.2871{\scriptsize$\pm$0.04} & \textbf{0.4633{\scriptsize$\pm$0.01}} & \textbf{0.4306{\scriptsize$\pm$0.01}} & \cellcolor{pink!28}\textbf{0.3951{\scriptsize$\pm$0.02}} \\
\midrule
\multicolumn{7}{c}{Qwen3.5-2B} \\
\midrule
Random                & 0.0374 & 0.3000 & 0.1856 & 0.1078 & 0.0850 & 0.1432 \\
Top-$k$ BM25          & 0.1973 & 0.7250 & 0.3807 & 0.6492 & 0.5491 & 0.5003 \\
Top-$k$ SBERT         & 0.2023 & 0.7786 & \textbf{0.5122} & 0.5875 & 0.5263 & 0.5214 \\
Top-$k$ BERT          & 0.2087 & 0.7607 & 0.4343 & 0.6260 & 0.5666 & 0.5193 \\
DPP-BERT              & 0.1845 & 0.8036 & 0.3683 & 0.6416 & 0.5642 & 0.5125 \\
Top-$k$ Qwen3-Emb.    & 0.1929 & 0.7786 & 0.4830 & 0.5879 & 0.5245 & 0.5134 \\
MLSM\textsuperscript{\ding{61}} & 0.2067 & 0.7750 & 0.4220 & 0.6644 & 0.5867 & 0.5310 \\
\rowcolor{pink!28}
CoRA (ours)           & \textbf{0.2187{\scriptsize$\pm$0.02}} & \textbf{0.8350{\scriptsize$\pm$0.05}} & 0.4576{\scriptsize$\pm$0.02} & \textbf{0.6677{\scriptsize$\pm$0.01}} & \textbf{0.5922{\scriptsize$\pm$0.02}} & \cellcolor{pink!28}\textbf{0.5542{\scriptsize$\pm$0.02}} \\
\bottomrule
\end{tabular}
\end{table}

\paragraph{Overall Performance on Textual Benchmarks.}
Table~\ref{tab:main-classification} and Table~\ref{tab:main-generation} reports ICL performance on five classification and five generation benchmarks with three downstream LLMs. CoRA achieves the highest average score in every model-by-task-category block, outperforming both static retrieval methods and MLSM. The improvements are observed across classification and generation settings, indicating that the output-conditioned retrieval basis can improve exemplar selection beyond input-similarity-based retrieval under diverse task formulations.

The comparison with MLSM is particularly informative because CoRA retains its representative-layer selection procedure while changing how the selected representations are used for retrieval. MLSM computes input-side similarities at the selected layers and optimizes query- or batch-specific aggregation weights to combine them. In contrast, CoRA uses paired candidate outputs during index construction to fit a fixed low-rank retrieval basis, after which retrieval for an unseen query requires only its input representation and the precomputed index. CoRA achieves higher average results than MLSM across all evaluated backbones and task categories, indicating a consistent advantage under the considered evaluation protocol and across the tested retrieval scenarios.

This comparison extends the two hypotheses from our preliminary conference work. Representative layers preserve complementary input-side signals associated with H1, while the output-derived conditioning basis incorporates task-specific structure motivated by H2. Within the evaluated setting, the results suggest that integrating these two signals in a fixed retrieval basis is more effective than aggregating task-agnostic layerwise similarity signals alone.

\paragraph{Transferability Across Downstream LLMs.}
CoRA achieves the strongest classification and generation averages with Llama-3.2-1B, MobileLLM-Pro, and Qwen3.5-2B, despite their differences in model architecture, scale, and deployment configuration. In particular, CoRA remains effective with MobileLLM-Pro, the 4-bit-quantized and most resource-constrained backbone in our evaluation. This result indicates that the retrieval basis can improve ICL even when the downstream model operates under a restricted memory footprint. CoRA also attains the strongest aggregate results with Qwen3.5-2B, a more recent and higher-capacity backbone, showing that its retrieval formulation remains compatible with stronger downstream models and is applicable across heterogeneous model settings.

The magnitude of improvement varies across models and task categories, which indicates that retrieval utility depends on both the downstream model and the task. Nevertheless, the consistent aggregate advantage across all three backbones supports the use of CoRA as a model-agnostic pre-inference retrieval component. It requires neither modification and fine-tuning of the textual encoder, nor retrieval-time interaction with the downstream model.

\paragraph{Task-wise Behavior and Failure Modes.}
CoRA performs strongly on several structured generation tasks, including GeoQuery, MTOP, and SMCalFlow, where the candidate outputs follow relatively regular compositional patterns. These tasks provide output-side representations that can serve as informative conditioning signals during index construction. However, the benefits of output-conditioned alignment are not uniform across all generation settings.

NL2Bash provides a representative challenging case. CoRA does not attain the best score on NL2Bash with any of the three downstream backbones, consistent with prior observations that this dataset is difficult for retrieval-based ICL~\cite{liu2024unraveling,ye2023compositional}. A plausible explanation is the many-to-many relation between natural-language intents and Bash commands. Functionally equivalent solutions can differ substantially in the utilities, flags, and command compositions that they use. Consequently, surface-form output representations may provide a less stable proxy for the functional relation relevant to retrieval. This interpretation is also consistent with Table~\ref{tab:layer_position}, where restricting CoRA to lower-layer representations improves NL2Bash performance relative to middle-, upper-, and CKA-selected layer configurations. The result suggests that pattern-level input cues are particularly important for this task and identifies high-variability output spaces as a limitation of the current output-conditioning formulation.

\paragraph{Performance on Multimodal Retrieval.}
\begin{table}[t]
\centering
\caption{Results on multimodal benchmarks. Best score of each task is boldfaced. Avg. denotes the average score across all five tasks.}
\label{tab:mllm_icl}
\begin{tabular}{lcccc >{\columncolor{blue!6}}c}
\toprule
Method & VQAv2 & OKVQA & VizWiz & MSCOCO & \textbf{Avg.} \\
\midrule
\multicolumn{6}{c}{OpenFlamingo-3B} \\
\midrule
Random       & 35.38 & 10.81 &  7.87 & 72.17 & 31.56\\
RICES\_image & 36.22 & 11.60 & 17.59 & 84.03 & 37.36\\
RICES\_text  & 36.57 & 12.53 & 10.91 & 74.13 & 33.53\\
MMICES       & 39.51 & 12.93 & 15.60 & 80.55 & 37.15\\
\rowcolor{pink!28}
CoRA-M (ours)         & \textbf{41.22{\scriptsize$\pm$0.01}} & \textbf{13.20{\scriptsize$\pm$0.02}} & \textbf{17.74{\scriptsize$\pm$0.01}} & \textbf{86.54{\scriptsize$\pm$0.03}} & \textbf{39.68{\scriptsize$\pm$0.01}} \\
\midrule
\multicolumn{6}{c}{Qwen3.5-2B} \\
\midrule
Random       & 50.95 & 25.37 &  9.67 & 87.55 & 43.38 \\
RICES\_image & 44.91 & 24.27 & 15.24 & 90.56 & 43.74 \\
RICES\_text  & 47.60 & 26.02 & 14.31 & 83.85 & 42.95 \\
MMICES       & 51.21 & 27.60 & 16.30 & 89.70 & 46.20 \\
\rowcolor{pink!28}
CoRA-M (ours)         & \textbf{53.92{\scriptsize$\pm$0.02}} & \textbf{29.03{\scriptsize$\pm$0.01}} & \textbf{18.54{\scriptsize$\pm$0.03}} & \textbf{90.96{\scriptsize$\pm$0.01}} & \textbf{48.11{\scriptsize$\pm$0.02}} \\
\bottomrule
\end{tabular}
\end{table}

Table~\ref{tab:mllm_icl} reports multimodal ICL results on three visual question-answering benchmarks and the MSCOCO image-captioning benchmark. CoRA-M achieves the strongest result for every dataset with both OpenFlamingo-3B and Qwen3.5-2B. The improvement is therefore consistent across two downstream VLMs and across both question answering and captioning settings.

The multimodal baselines considered here either retrieve demonstrations using a single modality or combine visual and textual signals through sequential filtering and ranking. In contrast, CoRA-M constructs a shared low-rank retrieval space in which visual candidate representations and output-side textual representations jointly condition the selected textual input features. The consistent gains in Table~\ref{tab:mllm_icl} are compatible with the benefit of this joint conditioning design for multimodal exemplar retrieval. Section~\ref{sec:abls} further examines the respective roles of visual representations and target-side textual conditioning through controlled ablations.

\subsection{Analysis of Retrieval Design}\label{sec:abls}
\paragraph{Importance of Target Conditioning.}
\begin{table}
\caption{Ablation study on key components of CoRA. \textit{w/o Fusion} uses only the final encoder layer. \textit{w/o Cond.} replaces output-conditioned projection with an unconditional one. Subscript $r$ denotes the number of retained dimensions. Default CoRA uses $r{=}256$.}
\label{tab:ablation}
\centering
\resizebox{\linewidth}{!}{
\begin{tabular}{l cccccccc}
\toprule
Dataset & Top-$k$ BERT & w/o Fusion & w/o Cond. & CoRA$_{r=32}$ & CoRA$_{r=64}$ & CoRA$_{r=128}$ & CoRA$_{r=256}$ & CoRA$_{r=512}$ \\
\midrule
SST-5 & 0.3878 & 0.3906 & 0.3969 & 0.3921 & 0.4042 & 0.4042 & 0.4133 & 0.4015 \\
QNLI  & 0.6577 & 0.6886 & 0.7003 & 0.7201 & 0.7415 & 0.7410 & 0.7322 & 0.7346 \\
WebQs & 0.3204 & 0.3278 & 0.3297 & 0.3254 & 0.3302 & 0.3258 & 0.3327 & 0.3302 \\
MTOP  & 0.5302 & 0.5199 & 0.5338 & 0.5139 & 0.5329 & 0.5302 & 0.5347 & 0.5242 \\
\bottomrule
\end{tabular}}
\end{table}
To isolate the contribution of output-derived conditioning, we compare CoRA with an unconditional variant (\textit{w/o Cond.}) that uses the same representative layers and low-rank dimensionality but replaces the output-conditioned operator in Eq.~\ref{eq_fitting} with the identity map, i.e., $\mathbf{P}_{\mathbf{C}}=\mathbf{I}$. Consequently, \textit{w/o Cond.} sets $\hat{\mathbf{H}}_x^{(\ell_j)}=\tilde{\mathbf{H}}_x^{(\ell_j)}$ and estimates its low-rank retrieval basis directly from the concatenated selected-layer input representations, without using candidate-output representations. As shown in Table~\ref{tab:ablation}, both projection-based variants outperform Top-$k$ BERT on all four datasets, indicating that low-rank transformation of the selected multi-layer features can improve downstream ICL performance beyond direct final-layer similarity retrieval. CoRA further outperforms \textit{w/o Cond.} on every dataset. This result indicates that conditioning the retrieval basis on paired candidate outputs provides information beyond input-side multi-layer representations alone and improves exemplar selection in the evaluated tasks.

\paragraph{Effects of Multi-Layer Representations and Layer Position.}
We examine whether the representative-layer construction improves retrieval beyond using the final encoder layer alone. In Table~\ref{tab:ablation}, full CoRA outperforms the \textit{w/o Fusion} variant on all four datasets. This comparison indicates that the selected layers provide complementary input-side information that is not consistently retained in the final-layer representation. Combining these representations therefore yields a more informative input feature space for the subsequent output-conditioned alignment.

\begin{table}[t]
\caption{Effect of layer position on CoRA retrieval performance. Each setting uses three encoder layers, either from a fixed depth range or selected by CKA.}
\label{tab:layer_position}
\centering
\begin{tabular}{lcccc}
\toprule
Dataset & Lower (1-3) & Middle (5-7) & Upper (10-12) & CKA-selected \\
\midrule
NL2Bash & \textbf{0.3419} & 0.2772 & 0.2637 & 0.3210 \\
GeoQuery & 0.5036 & 0.4429 & 0.5500 & \textbf{0.5714} \\
\bottomrule
\end{tabular}
\end{table}

The value of combining layers raises a related question, whether a fixed layer range is suitable across tasks. We address this question by comparing three fixed-depth layer groups with the CKA-selected layers used by CoRA. Table~\ref{tab:layer_position} shows contrasting layer preferences across the two generation tasks. NL2Bash performs best with lower-layer features, with performance gradually decreasing as the selected layers become deeper. In contrast, GeoQuery benefits from upper-layer features, while the CKA-selected configuration achieves the strongest result among the compared settings. These results show that no fixed layer range is uniformly preferable for exemplar retrieval across tasks with different input-output characteristics. The lower-layer preference on NL2Bash is consistent with the task's reliance on pattern-level input cues, whereas the GeoQuery result suggests that higher-level representations are useful for regular compositional mappings. CKA-based selection therefore provides a data-dependent alternative to a fixed depth range. It achieves the strongest result on GeoQuery and remains more effective than middle- and upper-layer configurations on NL2Bash.

\paragraph{Effect of Low-Rank Dimensionality.}
Table~\ref{tab:ablation} also evaluates the retained dimensionality $r$ of the retrieval basis. The effect of $r$ varies by task and does not follow a monotonic trend. Increasing the dimensionality from $r=32$ to an intermediate setting consistently improves performance across all four datasets, although the optimal value differs across tasks. QNLI reaches its strongest result at $r=64$, SST-5 at $r=128$ or $256$, and WebQs and MTOP at $r=256$. Further increasing the dimension to $r=512$ does not produce a consistent additional gain and instead reduces performance on several datasets. We therefore adopt $r=256$ as the default configuration, as it offers strong and broadly competitive results across the evaluated tasks while maintaining a compact $r$-dimensional index and query representation.

\paragraph{Effect of the Number of Retrieved Exemplars.}
\begin{figure}
    \centering
    \includegraphics[width=\linewidth]{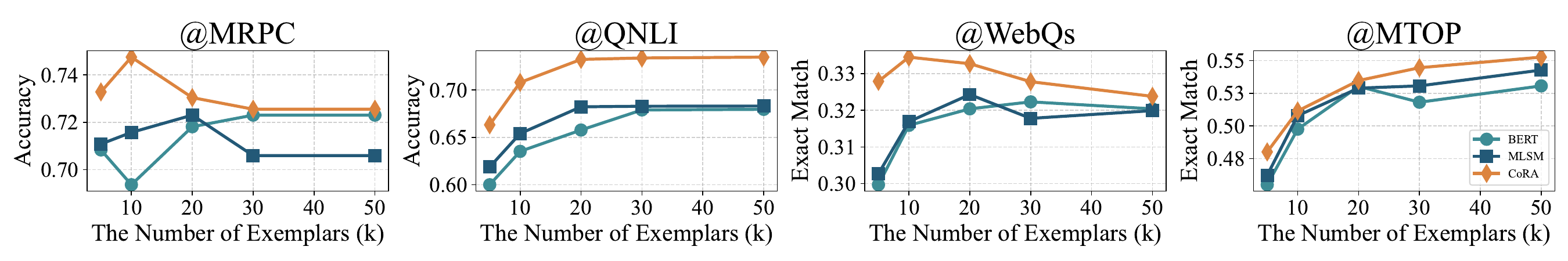}
    \caption{Performance comparison on four datasets with various number of exemplars $k$.}
    \Description{Ablation study on the number of exemplars.}
    \label{fig:various_k}
    
\end{figure}
Figure~\ref{fig:various_k} examines how the number of retrieved exemplars $k$ affects downstream ICL performance. The preferred context size differs substantially across tasks rather than following a shared pattern. MRPC and WebQs attain their strongest CoRA results with relatively small values of $k$, whereas QNLI and MTOP benefit from a larger context before performance plateaus or continues to improve. These task-specific trends indicate that increasing the number of demonstrations does not uniformly improve ICL performance.

Nevertheless, CoRA remains competitive across the evaluated context sizes. It achieves the strongest result at every displayed value of $k$ on all four datasets. These results indicate that the benefit of output-conditioned retrieval is not tied to a single demonstration budget. In practice, this permits $k$ to be selected according to the available prompt-length and downstream-inference budget while retaining the benefit of the learned retrieval basis.

\paragraph{Analysis of Multimodal Signal Placement and Target-Side Textual Conditioning.}
\begin{table}[t]
\caption{Ablation on visual signal placement and target-side textual conditioning in CoRA-M.}
\label{tab:mllm_abls}
\centering
\resizebox{\linewidth}{!}{
\begin{tabular}{lcccccc}
\toprule
Dataset & CoRA-M (shuffle-vis) & CoRA-T & CoRA-M (input-only) & CoRA-M (target-only) & CoRA-M (layerwise-target) & CoRA-M \\
\midrule
VQAv2  & 49.57 & 51.83 & 52.21 & 52.46 & 52.38 & \textbf{53.89}  \\
VizWiz & 15.03 & 17.41 & 17.87 & 17.13 & 17.48 & \textbf{18.56}  \\
\bottomrule
\end{tabular}}
\end{table}
Table~\ref{tab:mllm_abls} examines how visual features and target-side textual representations are incorporated into CoRA-M. The default CoRA-M configuration uses the final-layer representation of each candidate response together with its visual representation in the conditioning matrix, and additionally appends visual features to the retrieval representation. To isolate these design choices, \textit{input-only} adds visual features only to the retrieval representation, whereas \textit{target-only} adds them only to the conditioning matrix. Both retain the same final-layer target-text representation. \textit{CoRA-M (layerwise-target)} replaces the final-layer target-text representation with selected layerwise target features while retaining visual features on both sides. \textit{CoRA-T} applies the text-only formulation without visual features, and \textit{CoRA-M (shuffle-vis)} randomly permutes visual features across candidate examples while preserving their dimensionality and marginal distribution.

The full CoRA-M configuration achieves the strongest result on both VQAv2 and VizWiz. Adding visual features only to the retrieval representation improves over CoRA-T on both datasets, whereas adding them to the conditioning matrix improves VQAv2 but not VizWiz. The combined configuration is consistently stronger than either one-sided variant, indicating that the two placements provide complementary information. Moreover, CoRA-M outperforms \textit{CoRA-M (layerwise-target)} on both datasets, showing that the final-layer target representation is a more effective target-side textual conditioning choice than the selected layerwise alternative.

The shuffled-visual variant performs below CoRA-T on both datasets. By preserving visual feature dimensionality while disrupting their correspondence with candidate examples, this result indicates that CoRA-M gains depend on correctly aligned visual information rather than additional feature dimensions alone. Together, these ablations support the default CoRA-M design and motivate further analysis of its query-level retrieval behavior.

\paragraph{Retained Fitted Energy and Retrieval Performance.}
\begin{table}[t]
\centering
\caption{Retained fitted energy and downstream ICL performance under
a common rank budget of \(r=256\). All values of
\(\rho_r(\mathbf{W})\) are evaluated with respect to the same
output-conditioned fitted matrix \(\hat{\mathbf{H}}_x\).}
\label{tab:conditional_basis_analysis}
\begin{tabular}{llcc}
\toprule
Dataset & Basis
& \(\rho_r(\mathbf{W})\)
& Downstream metric \\
\midrule
MRPC
& Random Orthogonal & 0.116 & 0.7010 \\
& Unconditioned     & 0.648 & 0.7206 \\
& CoRA              & \textbf{0.927} & \textbf{0.7353} \\
\midrule
WebQS
& Random Orthogonal & 0.113 & 0.3022 \\
& Unconditioned     & 0.593 & 0.3223 \\
& CoRA              & \textbf{0.884} & \textbf{0.3327} \\
\bottomrule
\end{tabular}
\end{table}

Following Eq.~\ref{eq:retained_fitted_energy}, we measure the fitted energy retained by different bases under the same rank budget. All retained-energy ratios are computed with respect to the same output-conditioned fitted matrix \(\hat{\mathbf{H}}_x\). We compare CoRA's basis with an unconditioned basis formed by the top-\(r\) right singular vectors of \(\tilde{\mathbf{H}}_x\), and a random orthogonal basis of the same dimension.

Table~\ref{tab:conditional_basis_analysis} shows that CoRA retains the largest fraction of the output-conditioned fitted energy under the shared rank budget, consistent with Proposition~\ref{prop:conditional_basis}. The unconditioned basis retains substantially less fitted energy, indicating that the dominant directions of the raw input representations do not fully coincide with those identified through output conditioning. The random orthogonal basis exhibits the lowest retained energy, reflecting its lack of alignment with the fitted representation. The downstream results follow the same ordering on both datasets. Relative to the unconditioned and random alternatives, CoRA combines greater preservation of the fitted structure with higher ICL performance. This agreement connects the rank-constrained basis construction characterized in Proposition~\ref{prop:conditional_basis} with its empirical utility for exemplar retrieval under the same compact representation budget.

\subsection{Resource Cost and On-Device Deployment}\label{sec:exp_sys}
\paragraph{Aggregate Pre-Inference Retrieval Cost and Memory Scaling.}
\begin{figure}[t]
    \centering
    \begin{minipage}[t]{0.48\linewidth}
        \vspace{0pt}
        \centering
        \captionof{table}{Aggregate pre-inference retrieval latency cost (s) and peak CPU memory (MB) on MRPC under identical single-CPU settings. $cs$ denotes the fixed chunk size used by streaming under limited memory budgets.}
        \label{tab:computation}
        \vspace{2pt}
        \setlength{\tabcolsep}{3pt}
        \resizebox{\linewidth}{!}{
        \begin{tabular}{l ccccc}
        \toprule
        Method & Top-$k$ BERT & DPP-BERT & MLSM & CoRA & CoRA$_{cs=64}$\\
        \midrule
        Time & 47.33 & 56.20 & 54.36 & 47.94 & 57.23 \\
        Memory & 2315 & 2366 & 2609 & 3356 & 2396 \\
        \bottomrule
        \end{tabular}}
    \end{minipage}
    \hfill
    \begin{minipage}[t]{0.50\linewidth}
        \vspace{0pt}
        \centering
        \includegraphics[width=\linewidth]{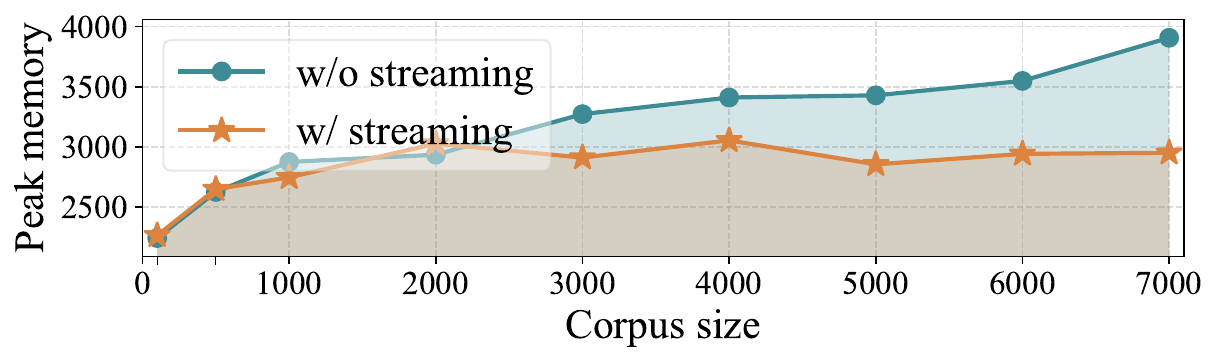}
        \vspace{-15pt}
        \captionof{figure}{Peak CPU memory (MB) versus corpus size during the pre-inference retrieval workflow, with and without chunked streaming.}
        \Description{Memory results.}
        \label{fig:corpus_size_peak_memory}
    \end{minipage}
    \vspace{-8pt}
\end{figure}
Table~\ref{tab:computation} reports aggregate pre-inference retrieval cost and peak CPU memory on MRPC under identical single-CPU settings. The reported time covers the complete workflow from representation computation through exemplar selection and prompt construction for the evaluated candidate pool and query set, but excludes downstream-model inference. It should therefore be interpreted as an aggregate pre-inference cost rather than query-time latency alone. Under this measure, CoRA requires $47.94$~s, which is comparable to Top-$k$ BERT and lower than DPP-BERT and MLSM. Unlike MLSM, CoRA performs no gradient evaluation or optimizer-based update during this workflow. When memory is constrained, chunked streaming with $|\mathcal{D}_b|=64$ reduces peak memory from 3,356~MB to 2,396~MB, placing it close to that of Top-$k$ BERT. This reduction increases aggregate pre-inference time to 57.23~s, reflecting the additional sequential passes over the candidate pool required by streaming.

Figure~\ref{fig:corpus_size_peak_memory} examines memory scaling as the candidate-pool size increases from 100 to 7,000 on QNLI under identical CPU-only settings. Without streaming, peak memory increases from approximately 2.3~GB to 3.9~GB over the measured range because the workflow materializes the full feature matrix $\tilde{\mathbf{H}}_x$ and conditioning matrix $\mathbf{C}$ before fitting the retrieval basis. With streaming, peak memory remains within a narrower range after approximately 3,000 candidates. The remaining growth is attributable to the persistent projected index $\mathbf{Z}\in\mathbb{R}^{n\times r}$, dataset objects, and label storage. This behavior is consistent with Eq.~\ref{eq_overhead}. The working memory for alignment is independent of $n$ for fixed chunk and feature dimensions, whereas persistent index storage retains an $\mathcal{O}(nr)$ dependence.

\paragraph{End-to-End On-Device Deployment on Raspberry Pi 5.}
\begin{table}[t]
\caption{End-to-end on-device results on Raspberry Pi 5 (8~GB RAM). Use \textit{Qwen3.5-0.8B} as the on-device model. Training-based methods run out of memory.}
\label{tab:ondevice}
\centering
\begin{tabular}{lccc >{\columncolor{pink!28}}c c}
\toprule
\multicolumn{6}{c}{Classification} \\
\midrule
Method & Top-$k$ BERT & DPP-BERT & MLSM & CoRA & EPR/CEIL/TTF \\ 
\midrule
MRPC  & 0.7059 & 0.6985 & 0.6961 & \textbf{0.7353} & \textcolor{red}{OOM} \\
QNLI  & 0.7368 & 0.7690 & 0.7318 & \textbf{0.7752} & \textcolor{red}{OOM} \\
\midrule
\multicolumn{6}{c}{Generation} \\
\midrule
Method & Top-$k$ BERT & DPP-BERT & MLSM & CoRA & EPR/CEIL/TTF \\
WebQs & 0.2106 & 0.2165 & 0.1969 & \textbf{0.2288} & \textcolor{red}{OOM} \\ 
MTOP  & 0.5839 & 0.5928 & 0.5996 & \textbf{0.6076} & \textcolor{red}{OOM} \\
\bottomrule
\end{tabular}
\end{table}
We evaluate textual CoRA in an end-to-end on-device setting using a Raspberry Pi~5 with 8~GB RAM, where the complete pre-inference retrieval workflow and downstream LLM inference are executed locally. To fit the device memory budget, we use \textit{Qwen3.5-0.8B} as the on-device inference model. As shown in Table~\ref{tab:ondevice}, CoRA achieves the strongest result on all four evaluated tasks relative to Top-$k$ BERT, DPP-BERT, and MLSM. EPR, CEIL, and TTF could not complete their corresponding task-adaptation and pre-inference retrieval workflows within the available device memory and are therefore reported as OOM. In contrast, CoRA completes the full textual retrieval-and-inference pipeline without gradient-based retriever adaptation. This experiment establishes end-to-end feasibility for textual CoRA on the evaluated device configuration.

\paragraph{Task-Adaptation and Pre-Inference Retrieval Cost Compared with Training-Based Methods.}
\begin{table}[t]
\centering
\caption{Efficiency-accuracy trade-off on MRPC. Time includes the full task adaptation/training and retrieval pipeline. Peak memory is reported per GPU. Unlike the CPU-only setting in the main text, this experiment is conducted on an 8$\times$RTX 4090 server because the training-based baselines require multi-GPU optimization.}\label{tab:mrpc_efficiency_tradeoff}
\begin{tabular}{lcccc}
\toprule
Method & GPU Setup & Peak Mem./GPU (MB) & Time (s) & Acc. \\
\midrule
EPR  & 8$\times$RTX 4090 & 11888 & 380 & 0.7549 \\
CEIL & 8$\times$RTX 4090 & 23480 & 506 & 0.7645 \\
TTF  & 8$\times$RTX 4090 & 13000 & 491 & 0.7525 \\
\midrule
CoRA & 1$\times$RTX 4090 & 2248  & 21  & 0.7353 \\
\bottomrule
\end{tabular}
\end{table}
We compare CoRA with three representative training-based demonstration-selection methods on MRPC: EPR, CEIL, and TTF. Because these methods require task-specific retriever adaptation before exemplar selection, Table~\ref{tab:mrpc_efficiency_tradeoff} reports the total cost from method-specific adaptation through pre-inference retrieval, rather than retrieval alone.

The training-based methods are executed on an $8\times$RTX~4090 server to accommodate their task-specific optimization procedures, whereas CoRA's closed-form workflow runs on a single RTX~4090. We report peak memory per GPU together with total wall-clock time for the complete adaptation-to-retrieval workflow. Under these execution settings, CoRA uses substantially less device memory and completes the workflow markedly faster than EPR, CEIL, and TTF. These savings are accompanied by accuracy differences of $1.96$, $2.92$, and $1.72$ percentage points, respectively, while CoRA reduces adaptation-to-retrieval time by factors of $18.1$, $24.1$, and $23.4$.

These results position CoRA as an efficiency-oriented alternative. CEIL attains the strongest MRPC accuracy when multi-GPU task-specific optimization is available. CoRA instead provides a gradient-free option for settings in which adaptation time, memory, or retriever-training resources are constrained.

\section{Conclusion and Limitations}\label{sec:conclusion}
We presented \emph{Conditional Retrieval Alignment} (CoRA), a gradient-free framework for task-conditioned exemplar retrieval in on-device ICL. CoRA uses paired candidate inputs and outputs to construct an output-conditioned retrieval space from complementary frozen-encoder layers. Candidate outputs are required only during offline index construction. At query time, CoRA retrieves exemplars using the query input and a compact precomputed low-rank index. This design provides task-conditioned retrieval without retriever fine-tuning, backpropagation, or target-model calls. We further derived an exact streaming construction procedure and extended the framework to multimodal retrieval. Experiments across textual and visual-language benchmarks, together with end-to-end deployment on a \emph{Raspberry Pi~5}, demonstrate the effectiveness and practical feasibility of this approach under edge resource budgets.

The current study also has several limitations. The current deployment study evaluates a software implementation on a Raspberry Pi~5. Extending the streaming construction to continuously evolving candidate pools, and designing specialized accelerators for better end-to-end energy consumption, remain important directions for future work.

\bibliographystyle{ACM-Reference-Format}
\bibliography{sample-base}

\end{document}